\documentclass[runningheads]{llncs}
\usepackage{graphicx}
\usepackage{amsmath,amssymb,mathtools}
\newtheorem{prop}{Proposition}

\let\doendproof\endproof
\renewcommand\endproof{~\hfill\qed\doendproof}

\newcommand{\dl}{\delta}
\newcommand{\ep}{\varepsilon}

\newcommand{\tht}{\theta}

\newcommand{\N}{\mathbb{N}}
\newcommand{\R}{\mathbb{R}}

\newcommand{\mbf}[1]{\mathbf{#1}}
\newcommand{\Jicomp}{J_{W_i, \mbf{x}}}
\newcommand{\Ji}{J_i}
\newcommand{\phxi}{\varphi_{\mbf{x}, i}}
\newcommand{\Rmn}{\R^{m\times n}}
\newcommand{\bfb}{\mbf{b}}
\newcommand{\bfh}{\mbf{h}}
\newcommand{\bfu}{\mbf{u}}
\newcommand{\bfv}{\mbf{v}}
\newcommand{\bfw}{\mbf{w}}
\newcommand{\bfx}{\mbf{x}}

\newcommand{\fth}{f_\theta}
\newcommand{\fith}[1]{f_{#1, \theta_{#1}}}

\newcommand{\fix}[1]{f_{#1, \bfx_{#1}}}

\begin{document}
\title{Training morphological neural networks with gradient descent: some theoretical insights}
\titlerunning{Training morphological networks with gradient descent: theoretical insights}
\author{Samy Blusseau\inst{1}\orcidID{0000-0003-0294-8172}}
%\author{Samy Blusseau}
%
\authorrunning{S. Blusseau}
\institute{Mines Paris, PSL University, Centre for mathematical
  morphology (CMM), Fontainebleau, France}
\maketitle              % typeset the header of the contribution
\begin{abstract}

  Morphological neural networks, or layers, can be a powerful tool to
  boost the progress in mathematical morphology, either on theoretical
  aspects such as the representation of complete lattice operators, or
  in the development of image processing pipelines. However, these
  architectures turn out to be difficult to train when they count more
  than a few morphological layers, at least within popular machine
  learning frameworks which use gradient descent based optimization
  algorithms. In this paper we investigate the potential and
  limitations of differentiation based approaches and back-propagation
  applied to morphological networks, in light of the non-smooth
  optimization concept of Bouligand derivative. We provide insights
  and first theoretical guidelines, in particular regarding
  initialization and learning
  rates.
  
  
\end{abstract}

\section{Introduction}
\label{sec:intro}

Morphological neural networks were introduced in the late
1980s~\cite{wilson1989morphological,davidson1990theory}, and have been
revisited in recent
years~\cite{charisopoulos2017perceptron,zhang2019max,mondal2020image,franchi2020deep,kirszenberg2021beyond}. With
the growing maturity of deep learning science, new exciting
perspectives seem to open and give hope for significant breakthroughs.

In image processing, with the development of libraries specialized in
morphological architectures \cite{velasco2020morpholayers}, where
basic as well as advanced operators are implemented, such as
geodesical reconstruction layers \cite{velascoforero2022fixed}, it is
now within reach to train end to end pipelines which include
morphological preprocessing and postprocessing, and to use the know how
of the morphological community to impose topological and geometrical
constraints inside deep networks.

Furthermore, morphological networks can help investigate in practice
the representation theory of lattice operators initiated by Georges
Matheron~\cite{matheron1975random,maragos1989representation,banon1993decomposition}.
Just as the universal approximation theorem for the multi-layer
perceptron, the representation theorem of lattice operators with
families of erosions and anti-dilations, is an existence one and is
asymptotic, but does not provide any algorithm to actually exhibit
such representations. Since these decompositions can be implemented as
morphological layers, we may hope to learn these representations
\emph{from data}.

Yet, the optimization of morphological architectures is still slow and
difficult. Despite the several contributions in this area,
\cite{mondal2020image,franchi2020deep,aouad2022binary}, architectures
including morphological layers are often quite shallow and do not
compete with the state of the art networks for image analysis. On the
one hand, it may be due to the Fr\'echet non differentiability of the
morphological layers, reason for several attempts to replace them by
smooth approximations
\cite{hermary2022learning,kirszenberg2021beyond}. On the other hand,
non smooth operations such as the Rectifier Linear Unit (ReLU) or the
max-pooling, are commonly used in successfully trained architectures,
while smooth morphological ones do not seem to solve all the
optimization issues.

In this paper we investigate the potential and limitations of training
morphological neural networks with differentiation-based algorithms
relying on back-propagation and the chain rule. In
Section~\ref{sec:mnn} we introduce morphological networks, and recall
in Section~\ref{sec:optim} the principles of gradient descent,
back-propagation and chain rule. Section~\ref{sec:bouligand_pc1}
presents the concept of Bouligand derivative, which is suited to
morphological layers. In Section~\ref{sec:optim_bouligand} we expose
the possibilities and issues of this framework within the chain-rule
paradigm, before concluding in Section~\ref{sec:conclusion}.

%masci13learning,
\section{Morphological networks}
\label{sec:mnn}

There is no universal definition of morphological neural networks, but
most architectures that are called so, are neural networks including
at least a morphological layer. In turn, a morphological layer is one
computing a morphological operation such as a dilation or an erosion,
or sometimes a (weighted) rank filter. In this paper we will focus on
dilation and erosion layers, composed with each other or with other
classical (dense or convolutional) layers.

%\paragraph{Dilation layers}
\textbf{Dilation layers.}
We will call dilation layer a function $\dl_W : \R^n\to\R^m$,
$n, m\in\N^*$, defined by

\begin{equation}
  \label{eq:dilation_layer}
  \dl_W:\mbf{x} = (x_1,\dots, x_n) \mapsto \left(\max\limits_{1\leq k\leq n} x_k + w_{i,k}\right)_{1\leq i\leq m}
\end{equation}
where the $w_{i,k}$ are the real valued coefficients of a matrix
$W\in\R^{m \times n}$, and the parameters (or \emph{weights}) of the
layer. Extended to the complete lattices $\bar{\R}^n$ and
$\bar{\R}^m$, where
$\bar{\R} \coloneqq \R\cup\lbrace -\infty, +\infty\rbrace$, $\dl_W$ is
a shift invariant morphological dilation
\cite{maragos13representations,blusseau2022morphological}. In
practical neural architectures the input and output of a layer are
usually represented as sets of vectors, called feature maps. In such a
setting, each output feature map would be the supremum of dilations
like $\dl_W$, of the input feature maps. By reshaping the set of input
feature maps into one input vector, and the set of output ones into
one output vector, we get the equivalent formulation
\eqref{eq:dilation_layer}, simpler to analyze.

\textbf{Erosion layers.}
Similarly, we will call erosion layer a function
$\ep_W : \R^m\to\R^n$, $n, m\in\N^*$, defined by
\begin{equation}
  \label{eq:erosion_layer}
  \ep_W:\mbf{x} = (x_1,\dots, x_m) \mapsto \left(\min\limits_{1\leq k\leq m} x_k - w_{k,j}\right)_{1\leq j\leq n}.
\end{equation}
Note that the sign ``$-$'' and the transposition ($w_{k,i}$ instead of
$w_{i,k}$) in the definition are meant to make $(\ep_W, \ \dl_W)$ a
morphological adjunction.

\textbf{Morphological networks.} As said earlier, in this paper any
neural network including at least a morphological layer is considered
a morphological network. This includes sequential compositions of
dilations and erosions layers, supremum of erosion layers, infimum of
dilation layers, and composition with classical layers (linear
operators followed by a non-linear activation function). This also
includes anti-dilations and anti-erosions, which are of the kind
$\mbf{x}\mapsto \dl_W(-\mbf{x})$ and $\mbf{x}\mapsto
\ep_W(-\mbf{x})$. Note however that the composition
$\dl_{A} \circ\dl_{B}$ of two dilation layers can be considered as one
dilation layer $\dl_{C}$ where $C\in\R^{m\times n}$ is the max-plus
matrix product of $A\in\R^{m\times p}$ by $B\in\R^{p\times n}$,
\begin{equation}
  \label{eq:mat_product}
  C_{ij} =  \max_{1\leq k \leq p} A_{ik}+ B_{kj}, \;\;\; 1\leq i \leq m,\;\; 1\leq j \leq n.
\end{equation}
Furthermore, the pointwise maximum of $l$ dilation layers
$\dl_{W_1}, \dots, \dl_{W_l}$ (where all $W_i$s are the same size), is
also equivalent to one dilation layer $\dl_{W^*}$ where $W^*$ is the
pointwise maximum the matrices $W_i$s.

Similarly, in theory it is pointless to compose or take the minimum of
erosion layers, since such operators can be represented (and learned)
as one erosion layer.

\section{Optimization with gradient descent}
\label{sec:optim}

Let us consider a classic neural network setting where a function
$f_{\theta}:\R^n \to \R^+$ depending on a parameter
$\theta = [\theta_1, \dots, \theta_L]$ is a composition of $L$
functions
\begin{equation}
  \label{eq:f_theta}
  f_{\theta} \coloneqq f_{L,\theta_L}\circ f_{L-1,\theta_{L-1}} \circ \dots\circ f_{1,\theta_{1}},
\end{equation}
each $f_{k,\theta_{k}}$ depending on its own parameter
$\theta_k\in\R^{p_k}$ and mapping $\R^{n_k}$ to $\R^{n_{k+1}}$, with
$n_1 = n$ and $n_L=1$ (we include the loss function as part of the
last layer). Typically, we would like to find a parameter $\theta$
which minimizes the expectation $\mathbb{E}(f_{\theta}(X))$ where $X$
is a random variable that models the distribution of the data we want
to process\footnote{Recall that $f_\theta$ is real valued since we
  include the loss in the last layer $f_{L,\theta_{L}}$.}. In practice
this can be done by applying $f_{\theta}$ to samples $x_1, \dots, x_N$
of $X$ and iteratively update $\theta\gets \theta + \Delta \theta$ in
a way that decreases the function at the current sample,
$f_{\theta + \Delta\theta}(x_i) \leq f_{\theta}(x_i)$. Hence the
change $\Delta \theta$ that is looked for is a \emph{descent
  direction}.

\vspace{-.3cm}
\subsection{Gradient descent}
\label{sec:gradient_descent}

Where it exists, the gradient of a function $g:\R^n\to\R$ precisely
provides a descent direction. Indeed if $g$ is
Fr\'echet-differentiable\footnote{The Fr\'echet derivative is just the
  usual derivative, which is a linear function, like
  $h\mapsto \langle \nabla g(x), h \rangle$ in \eqref{eq:g_x_h}.} at $x\in\R^n$, then
\begin{equation}
  \label{eq:g_x_h}
    \forall h\in\R^n, \; \forall \eta \geq 0, \;\;\; g\big(x + \eta h\big) = g(x) + \eta \big(\langle \nabla g(x), h \rangle +\epsilon (\eta)\big)
\end{equation}
where $\langle \cdot,\cdot\rangle$ is the inner product in $\R^n$ and
$\epsilon$ is a function that goes to zero when $\eta$ goes to
zero. Hence, if $\nabla g(x)\neq 0$, for $\eta$ sufficiently small
$|\epsilon(\eta)|< \Vert\nabla g(x)\Vert^2$ and therefore
$g\big(x - \eta \nabla g(x)\big) < g(x)$, for which $-\nabla g(x)$ is
called a descent direction of $g$ at $x$. Equation \eqref{eq:g_x_h}
also implies that any $h\in\R^n$ such that
$\langle \nabla g(x), h\rangle\leq 0$ is a descent
direction. Furthermore, it shows $-\nabla g(x)$ is the \emph{steepest}
descent direction: given $\eta>0$ sufficiently small, any unit vector
$v$ verifies
$g\big(x - \eta \frac{\nabla g(x)}{\Vert\nabla g(x)\Vert}\big) \leq g\big(x
+ \eta v\big)$.

These results can be applied to the function
$g:\theta \mapsto f_{\theta}(x)$ for a fixed sample $x$, provided $g$
is a Fr\'echet-differentiable (also called F-differentiable) function
of $\theta$. In that case we will note
$\nabla_{\theta} f_{\theta}(x) \coloneqq \nabla g$.

\subsection{Back propagation and the chain rule}
\label{sec:back_prop}

To compute $\nabla_{\theta} f_{\theta}(x)$, it is sufficient to
compute each $\nabla_{\theta_i} f_{\theta}(x)$, which can also be
noted $\frac{\partial f_{\theta}(x)}{\partial \theta_i}$, and is the
gradient of the function $g_i: \theta_i \mapsto f_{\theta}(x)$, $x$
and the other parameters $\theta_j, j\neq i$, being fixed. Indeed, the
gradient with respect to $\theta$ is the concatenation of the
gradients with respect to the $\theta_i$s,
$\nabla_{\theta} f_{\theta}(x) = [\nabla_{\theta_1} f_{\theta}(x),
\dots, \nabla_{\theta_L} f_{\theta}(x)]$.

% To obtain these, the so called ``chain rule'' is applied. The chain
% rule paradigm assumes that every layer is Fr\'echet differentiable
% with respect to its input variable and with respect to its parameter.

To obtain these, the so called ``chain rule'' is applied, involving
the (Fr\'echet) derivative of each layer with respect to its input
variable and its derivative with respect to its parameter. The
derivatives with respect to the input variables tell earlier layers
(i.e. the layers that are closer to the input) how they should change
their output values to eventually decrease the whole function
$f_{\theta}(x)$. They play a role of message passing to earlier
layers. The derivative with respect to a layer's parameter tells how
to update this parameter in order to comply with the instruction
received from later layers (that is, layers closer to the output).

More formally, we can see this in the case of $\fth$ as defined in
\eqref{eq:f_theta}. We note $\bfx_1\coloneqq x$ the input variable of
$\fith{1}$ (and therefore $\fth$), and
$\bfx_{k+1} \coloneqq \fith{k}(\bfx_k)$, $1\leq k\leq L-1$. For fixed
$\tht_k, \bfx_k$, we denote by $\fith{k}^\prime(\bfx_k ; \ \cdot \ )$
and $\fix{k}^\prime(\tht_k ; \ \cdot \ )$ the Fr\'echet derivatives of
the $k$-th layer with respect to its input variable and parameter
respectively. Then the chain rule algorithm can be summarized as
follows (see Figure \ref{fig:backprop}).
\begin{figure}
  \centering
  \includegraphics[width=.8\textwidth]{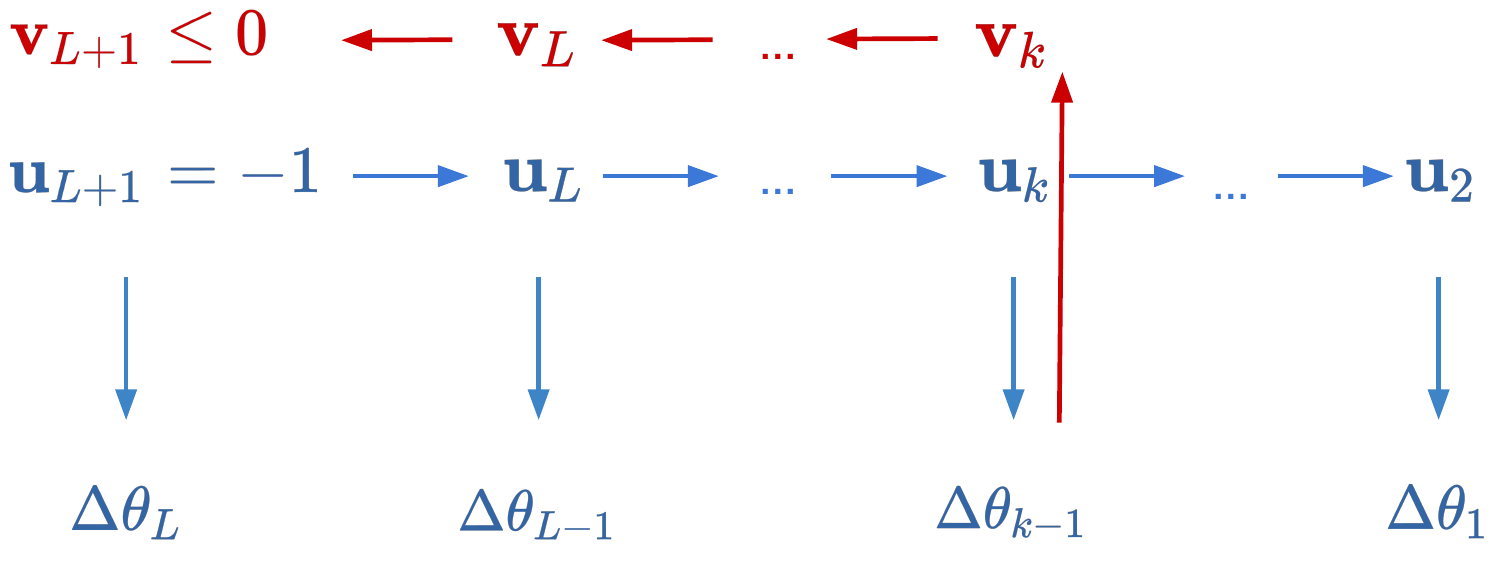}
  \caption{illustration of the chain rule algorithm.}
  \label{fig:backprop}
\end{figure}

\textbf{Initialize the message $\bfu_{L+1}$:} Since we want to
decrease $\fth$, the first target direction to be passed on to layer
$L$ is $\bfu_{L+1} = -1$.

\textbf{Update $\tht_k$, given $\bfu_{k+1}$:} Move $\tht_k$ in the
direction %$\Delta\tht_k$
\begin{equation}
  \label{eq:Delta_theta_k}
  \Delta\tht_k \coloneqq \arg\max_{\Vert\bfh\Vert=1} \langle \fix{k}^\prime(\tht_k ; \bfh), \bfu_{k+1}\rangle.
\end{equation}

\textbf{Pass on message $\bfu_{k}$ , given $\bfu_{k+1}$:} If $k\geq 2$
pass to layer $k-1$ the target direction
\begin{equation}
  \label{eq:message_u_k}
  \bfu_k \coloneqq \arg\max_{\Vert\bfh\Vert=1} \langle \fith{k}^\prime(\bfx_k ; \bfh), \bfu_{k+1}\rangle.
\end{equation}

Both problems \eqref{eq:Delta_theta_k} and \eqref{eq:message_u_k} are
easily solved using $\fix{k}^*(\tht_k ; \ \cdot \ )$ and
$\fith{k}^*(\bfx_k ; \ \cdot \ )$, the adjoint operators to the
derivatives $\fix{k}^\prime(\tht_k ; \ \cdot \ )$ and
$\fith{k}^\prime(\bfx_k ; \ \cdot \ )$ respectively:
\begin{equation}
  \label{eq:solution_chain_rule}
  \begin{array}{lcr}
    \Delta\tht_k = \frac{\fix{k}^*(\tht_k ; \ \bfu_{k+1} )}{\Vert\fix{k}^*(\tht_k ; \ \bfu_{k+1} )\Vert} \;\;
    & \text{ and }
    & \;\; \bfu_{k}  = \frac{\fith{k}^*(\bfx_k ; \ \bfu_{k+1} )}{\Vert\fith{k}^*(\bfx_k ; \ \bfu_{k+1} )\Vert}.
  \end{array}
\end{equation}
These solutions do not ensure a change of the output value of layer
$k$ in the direction $\bfu_{k+1}$, but they do guarantee
\begin{equation}
  \label{eq:property_chain_rule}
  \begin{array}{lcr}
    \langle \fix{k}^\prime(\tht_k ; \Delta\tht_k), \bfu_{k+1}\rangle \geq 0 \;\;
    & \text{ and }
    & \;\; \langle \fith{k}^\prime(\bfx_k ; \bfu_k), \bfu_{k+1}\rangle \geq 0.
  \end{array}
\end{equation}

\textbf{Answer $\bfv_{k}$ to message $\bfu_{k}$:} Then, when layer
$k-1$ ($k\geq 2$) updates its parameter in the direction
$\Delta\tht_{k-1}$, its output does not move in the target direction
$\bfu_k$, but in the direction
$\bfv_k \coloneqq \fix{k-1}^\prime(\tht_{k-1}; \Delta\tht_{k-1})$,
which ``only'' verifies $\langle \bfv_k, \bfu_k \rangle \geq 0$,
according to \eqref{eq:property_chain_rule}. Therefore, the output of
layer $k$ moves in the direction
$\bfv_{k+1} \coloneqq \fith{k}^\prime(\bfx_{k}; \bfv_k)$ instead of
$\fith{k}^\prime(\bfx_{k}; \bfu_k)$, and so on. The linearity of
$\fith{k}^\prime(\bfx_{k}; \ \cdot \ )$ ensures the crucial following
property
\begin{equation}
  \label{eq:u_k_v_k_u_k1_v_k1}
  \begin{array}{ccc}
    \left\lbrace
    \begin{array}{c}
      \langle \fith{k}^\prime(\bfx_k ; \bfu_k), \bfu_{k+1}\rangle \geq 0\\
      \langle \bfv_k, \bfu_k \rangle \geq 0
    \end{array}\right.
    & \Rightarrow & \langle \bfv_{k+1}, \bfu_{k+1} \rangle  \coloneqq \langle \fith{k}^\prime(\bfx_{k}; \bfv_k), \bfu_{k+1} \rangle \geq 0.
  \end{array}
\end{equation}
Hence, as soon as \eqref{eq:property_chain_rule} and
\eqref{eq:u_k_v_k_u_k1_v_k1} hold for layer $k-1$ and later layers,
the property $\langle \bfv_k, \bfu_k \rangle \geq 0$, triggered by the
update $\Delta\tht_{k-1}$, propagates and eventually yields
$\langle \bfv_{L+1}, \bfu_{L+1} \rangle \geq 0$, i.e.
$\bfv_{L+1} \leq 0$, meaning the output of $f_{\tht}$ is decreased.

This quick reminder of the chain rule mechanism highlights that the
layer derivatives have two goals: optimal message passing and optimal
parameter update based on the message passed by later
layers. Therefore, in the case of non Fr\'echet-differentiable layers,
like dilation and erosion layers, we may investigate if these two
targets, namely properties \eqref{eq:property_chain_rule} and
\eqref{eq:u_k_v_k_u_k1_v_k1}, can still be met somehow. In the next
sections we will see that morphological layers are differentiable in
the more general sense of the Bouligand differentiability, which makes
this notion worth analyzing in the perspective of optimization with
gradient-descent-like algorithms.

\section{The Bouligand derivative}
\label{sec:bouligand_pc1}

The Bouligand derivative has been introduced in the nonsmooth analysis
literature \cite{robinson1987local,scholtes2012introduction}. It is a
directional derivative that provides a first order approximation of
its function in all directions. Formally, given a function
$g:\R^n\to\R^m$ and $x\in\R^n$, if for every $y\in\R^n$ the limit
\begin{equation}
  \label{eq:diectional_limit}
  g^{\prime}(x ; y) \coloneqq \lim_{\alpha\to 0, \alpha>0} \frac{g(x+\alpha y) - g(x)}{\alpha}
\end{equation}
exists, then $g$ is directionally differentiable at $x$ and
$g^{\prime}(x ; . \ )$ is called its directional derivative at $x$. If
additionally for any $h\in\R^n$
\begin{equation}
  \label{eq:first_order_approx}
  g(x+h) = g(x) + g^{\prime}(x; h) + o_{0}(h)
\end{equation}
then $g$ is said to be Bouligand differentiable (or B-differentiable)
at $x$, and $g^{\prime}(x ; . \ )$ is its Bouligand derivative, also
called B-derivative\footnote{Recall that $o_{0}(h)$ denotes
  $h\cdot\epsilon(h)$ where $\epsilon$ is any function that goes to
  zero when $h$ goes to zero.}. If $g$ is B-differentiable at every
$x\in\R^n$, then it simply said B-differentiable.

Fr\'echet differentiablilty implies B-differentiability, but what
makes the latter more general than the former is that the B-derivative
does not need to be a linear function.  If $g^{\prime}(x ; . \ )$ is a
linear function, then $g$ is Fr\'echet differentiable at $x$, and
$g^{\prime}(x ; . \ )$ is its Fr\'echet derivative at that point.

The B-derivative has nice properties similar to the Fr\'echet
derivative, in particular \cite{scholtes2012introduction}:
\begin{itemize}
\item \textbf{Positive homogeneity:}
  $g^\prime(x;\lambda y) = \lambda g^\prime(x ; y)$ for any $\lambda \geq 0$.

\item \textbf{Chain rule:} if $f:\R^n \to \R^m$ and $g:\R^p \to \R^n$
  are continuous and B-differentiable at $x\in\R^p$ and $g(x)$
  respectively, then $f\circ g$ is B-differentiable at $x$ and
  \begin{equation}
    \label{eq:chain_rule}
    \left(f\circ g\right)^\prime (x; y) = f^\prime \left( g(x) ; g^\prime(x; y) \right)
  \end{equation}
\item \textbf{Linearity of $f \mapsto f^\prime(x; .)$:} if
  $f:\R^n \to \R^m$ and $g:\R^n \to \R^m$ are continuous and
  B-differentiable at $x\in\R^n$, then so is $\alpha f + \beta g$ for
  any $\alpha, \beta\in\R$ and
  \begin{equation}
    \label{eq:linearity_b_derivative}
    (\alpha f + \beta g)^\prime(x; y) = \alpha f^\prime(x; y) + \beta g^\prime(x; y).
  \end{equation}
\item \textbf{Derivative of components:} $g : \R^n\to\R^m$ is
  B-differentiable at $x$ if and only if each of its components $g_i:\R^n\to\R$
  is, and in this case
  \begin{equation}
    \label{eq:b_derivative_components}
    g^\prime(x ; y) = \big(g_1^\prime(x ; y), \dots, g_m^\prime(x ; y)\big).
  \end{equation}
\end{itemize}

As we will see, the dilation and erosion layers are continuous and
B-differentiable functions of both their input variables and
parameters, as well as all the usual neural layers. Therefore, a
neural network $f_\theta(x)$ is a continuous and B-differentiable
function of its parameter $\theta\in\R^p$ for a fixed $x$. Noting
$g:\theta \mapsto f_\theta(x)$ we have for $h\in\R^p$ and any
$\eta>0$,
\begin{equation}
  \label{eq:g_theta_h}
  g(\theta+\eta h) = g(\theta) + \eta \big( g^\prime(\theta ; h) + \epsilon(\eta)\big)
\end{equation}
where $\epsilon$ is a function that goes to zero when $\eta$ goes to
zero. Hence, we are left with finding in which direction $h$ we need
to move the parameter $\theta$ in order to ensure
$g(\theta+\eta h) < g(\theta)$ for $\eta$ sufficiently small. Whereas
this was straightforward when
$g^\prime(\theta ; h) = \langle \nabla g(\theta), \ h \rangle$ in
Equation \eqref{eq:g_x_h}, the problem is open when
$g^\prime(\theta ; \cdot)$ is not linear. The purpose of the next
section is to focus on this problem in the case of morphological
neural networks.

\section{Optimization with the Bouligand derivative}
\label{sec:optim_bouligand}

\subsection{Derivatives of the morphological layers}
\label{sec:deriv_morpho_layers}

The Bouligand derivatives of the dilation and erosion layers with
respect to their input values and parameters, are well known in the
nonsmooth optimization literature \cite{scholtes2012introduction},
since they are easy examples of \emph{piecewise affine functions} for
which formulas exist. Here we provide some details of their
computation, that will matter in addressing the problem stated in the
previous section. We focus on the dilation layers, the case of
erosions being analogous.

With the same notations as in Section \ref{sec:mnn}, we denote by
$\mbf{x}\in\R^n$ and $W\in\R^{m\times n}$ the input vector and
parameter matrix of a dilation layer. We will note $\dl_W(\mbf{x})$ to
clarify that we are considering a function of $\mbf{x}$ with fixed
parameter $W$, and $\dl_{\mbf{x}}(W)$ for a function of $W$ with fixed
$\mbf{x}$.

\subsubsection{Derivative with respect to $W$}
\label{sec:derivative_w}

An interesting property of $\dl_{\bfx}$ is that, if we move away from
$W$ in the direction $H\in\Rmn$, with a sufficiently small step
$\eta \geq 0$, $\dl_{\bfx}(W+\eta H)$ shows an exact affine
behaviour. Proposition \ref{prop:affine_interval} below provides a
sufficient and necessary condition on the step $\eta$ for this to
hold. It will also provide the Bouligand derivative of $\dl_{\bfx}$.

Given a fixed $\mbf{x}\in\R^n$ and a variable $W\in\R^{m\times n}$ we
note $\dl_{\mbf{x}}(W) = \big( \phxi(W)\big)_{1\leq i\leq m}$ with
\begin{equation}
  \label{eq:phi_x_i}
  \phxi(W) \coloneqq \max_{1 \leq j\leq n} w_{ij} + x_j.
\end{equation}
Additionally, for each index $1\leq i\leq m$, let us note
\begin{equation}
  \label{eq:J_wi_x}
  \Jicomp \coloneqq \left\lbrace j\in\lbrace 1,\dots, n\rbrace, \; \phxi(W) = w_{ij}+x_j\right\rbrace
\end{equation}
the set of indices where the maximum is achieved in $\phxi(W)$. When
$W$ and $\bfx$ will be clear from the context, we shall just denote it
by
$\Ji$.\\
Let $H\in \Rmn$. Then for each $1\leq i\leq m$, we also introduce the
set
\begin{equation}
  \label{eq:Ki_1}
  K_i \coloneqq \left\lbrace k\in\lbrace 1,\dots, n\rbrace, \; h_{ik} > \max_{j\in\Ji} h_{ij}\right\rbrace.
\end{equation}
Then we have the following result:
\begin{prop}
  \label{prop:affine_interval}
  For fixed $W , H\in \Rmn$ and $\bfx\in\R^n$, let $\phxi$, $\Ji$ and
  $K_i$ as defined by \eqref{eq:phi_x_i}, \eqref{eq:J_wi_x} and
  \eqref{eq:Ki_1} respectively for $1\leq i\leq m$. Let
  \begin{equation}
    \label{eq:epsilon_i_general}
    \epsilon_i = \min_{k\in K_i}\frac{\phxi(W) - (w_{ik}+x_k)}{h_{ik} - \max_{j\in\Ji} h_{ij}}, \;\;\; 1\leq i\leq m,
  \end{equation}
  and $\epsilon = \min_{1\leq i\leq m} \epsilon_i$.
  Then, for any $\eta\in\R^+$ we have
  \begin{equation}
    \label{eq:affine_behaviour_general}
    \eta \in [0, \epsilon] \iff \dl_{\mbf{x}}(W+\eta H) = \dl_{\mbf{x}}(W) +\eta \left( \max_{j\in \Ji} h_{ij} \right)_{1\leq i \leq m}.
  \end{equation}
\end{prop}

\begin{proof}[Proposition \ref{prop:affine_interval}]
  Let $\eta\in\R^+$, and let us note
  $\bfb\coloneqq \left( \max_{j\in \Ji} h_{ij} \right)_{1\leq i \leq
    m}$. Then $\dl_{\mbf{x}}(W+\eta H) = \dl_{\mbf{x}}(W) +\eta \bfb$
  if and only if $\phxi(W+\eta H) = \phxi(W) +\eta b_i$ for all
  $1\leq i\leq m$.
  Now, the left-hand term writes
  \begin{equation}
    \label{eq:prop1_intermed0}
    \phxi(W+\eta H) = \max\left(
      \max_{j\in\Ji} w_{ij}+x_j+\eta h_{ij}, \max_{k\notin\Ji} w_{ik}+x_k+\eta h_{ik}\right).
  \end{equation}
  Since by definition $\phxi(W) = w_{ij}+x_j$ for any $j\in\Ji$, we get
  \begin{equation}
    \label{eq:prop1_intermed}
    \phxi(W+\eta H) = \max\left( \phxi(W)
      + \eta b_i, \; \max_{k\notin\Ji} w_{ik}+x_k+\eta h_{ik}\right).    
  \end{equation}
  Therefore $\phxi(W+\eta H) = \phxi(W) +\eta b_i$ if and only if for
  any $k\notin \Ji$, $\phxi(W) + \eta b_i\geq w_{ik}+x_k+\eta h_{ik}$
  which is equivalent to $\eta\leq \epsilon_i$, and the result
  follows.
\end{proof}

Given the definitions of Section~\ref{sec:bouligand_pc1},
Proposition~\ref{prop:affine_interval} readily shows that
$\dl_{\mbf{x}}$ has a directional derivative in any direction. By a
similar reasoning, one can show that
$\dl_{\mbf{x}}(W+H) = \dl_{\mbf{x}}(W) + \left( \max_{j\in \Ji} h_{ij}
\right)_{1\leq i \leq m}$ as soon as all $|h_{ij}|$ are small
enough\footnote{Take for example
  $\Vert H\Vert_\infty < \frac{1}{2} \min_{1\leq i\leq m}\min_{k\notin \Ji}
  \phxi(W) - (w_{ik} + x_k)$.}. Therefore $\dl_{\mbf{x}}$ is Bouligand
differentiable everywhere and its B-derivative is
% Given the definitions of Section~\ref{sec:bouligand_pc1},
% Proposition~\ref{prop:affine_interval} readily gives that
% $\dl_{\mbf{x}}$ is Bouligand differentiable everywhere and its
% B-derivative is
\begin{equation}
  \label{eq:deriv_dil_x_W}
  \dl_{\mbf{x}}^\prime(W ; H) = \left( \max_{j\in \Ji} h_{ij} \right)_{1\leq i \leq m}.
\end{equation}
Furthermore, with the notations of Proposition
\ref{prop:affine_interval},
\begin{equation}
  \label{eq:delta_x_W_affine}
  \eta \in [0, \epsilon] \iff \dl_{\mbf{x}}(W+\eta H) = \dl_{\mbf{x}}(W) +\eta\dl_{\mbf{x}}^\prime(W ; H).
\end{equation}

It appears that for any $W$ such that $\Ji = \lbrace j_i\rbrace$ is a
singleton for each $1\leq i\leq m$ (the maximum is achieved only once
for each $\phxi$), $\dl_{\mbf{x}}^\prime(W ; H)$ is a linear function
since
$\max_{j\in \Ji} h_{ij} = h_{ij_i} = \langle H_{i,:}, e_{j_i}\rangle$,
where $e_{j_i}$ is the vector with a one at index $j_i$ and zeros
elsewhere. Hence in that case $\dl_{\mbf{x}}$ is Fr\'echet
differentiable. One can check\footnote{Indeed the set of matrices for
  which the maximum in $\phxi(W)$ is achieved more than once, for a
  given $i$, is of zero Lebesgue measure.} that this happens for
almost every $W$.

\subsubsection{Derivative with respect to $\mbf{x}$}
\label{sec:derivative_x}

With the same approach as previously, one can show that $\dl_W$ is
B-differentiable with respect to $\mbf{x}$, and its B-derivative is,
for all $\mbf{h}\in\R^n$,
\begin{equation}
  \label{eq:deriv_dil_W_x}
  \dl_{W}^\prime(\mbf{x} ; \mbf{h}) = \left( \max_{j\in \Ji} h_j \right)_{1\leq i \leq m},
\end{equation}
Furthermore, for a fixed $\bfh\in\R^n$, changing only $h_{ij}$ and
$h_{ik}$ for $h_{j}$ and $h_{k}$ in \eqref{eq:Ki_1} and
\eqref{eq:epsilon_i_general}, it comes that for any $\eta\in\R^+$
\begin{equation}
  \label{eq:dil_W_x_h}
  \eta \in [0, \epsilon] \iff \dl_{W}(\mbf{x} + \eta\mbf{h}) = \dl_{W}(\mbf{x}) + \eta\dl_{W}^\prime(\mbf{x} ; \mbf{h}).
\end{equation}
Again, $\dl_{W}^\prime(\mbf{x} ; \mbf{h})$ is a linear function of
$\mbf{h}$ as soon as the maximum is achieved only once in
$\lbrace w_{ij} + x_j, 1\leq j\leq n \rbrace$, i.e.
$\Ji = \lbrace j_i \rbrace $, for each $1\leq i \leq m$. In that case
$\max_{j\in \Ji} h_j = h_{j_i} = \langle \mbf{h}, e_{j_i} \rangle$,
hence $\dl_{W}^\prime(\mbf{x} ; \mbf{h}) = E \mbf{h}$, where $E$ is
the matrix whose rows are the $e_{j_i}$s. Again, this case holds for
almost every $\mbf{x}$.

\subsection{Updating the parameters}
\label{sec:updating_params}

\subsubsection{Problem setting.}
\label{sec:pbm_setting_param_update}

Let us focus on the update of the parameter $W$ of the dilation layer
$\dl_W : \R^n\to\R^m$. In the context of the chain rule, we assume
that later layers (those closer to the output) have transmitted an
instruction direction $\mbf{u}\in\R^m$, and $\dl_W$ is supposed to
modify its parameter $W \gets W + \Delta W $ so that
$\dl_{\mbf{x}}(W+\Delta W)-\dl_{\mbf{x}}(W)$ maximizes the inner
product with $\mbf{u}$. More formally, just as in
Eq. \eqref{eq:Delta_theta_k}, we want to solve
\begin{equation}
  \label{eq:Delta_W}
  \Delta W = \arg\max_{\Vert H\Vert = 1} \left\langle \dl_{\bfx}^\prime(W; H), \bfu \right\rangle,
\end{equation}
where $\Vert\cdot\Vert$ denotes the Frobenius norm and, this time, we
consider the Bouligand derivative \eqref{eq:deriv_dil_x_W} computed
earlier. The reason for which we are faced with the same problem as
with Fr\'echet derivative, is that the Bouligand one also provides the
first order approximation \eqref{eq:first_order_approx}. Without loss
of generality, we assume $\Vert\bfu\Vert = 1$.

Solving \eqref{eq:Delta_W} does not seem straightforward but an
attempt could start by noticing that
\begin{equation}
  \label{eq:upper_bound_deriv_dil_x_W_H}
  \Vert\dl_{\bfx}^\prime(W; H)\Vert^2 = \sum_{i=1}^m \left(\max_{j\in \Ji} h_{ij}\right)^2\leq \sum_{i=1}^m\sum_{j=1}^n h_{ij}^2 = \Vert H\Vert^2
\end{equation}
hence $\Vert\dl_{\bfx}^\prime(W; H)\Vert \leq 1$ for $\Vert H\Vert=1$, therefore
$\left\langle \dl_{\bfx}^\prime(W; H), \bfu \right\rangle \leq \Vert u\Vert =
1$. This upper-bound is obviously achieved when
$\dl_{\bfx}^\prime(W; \ \cdot \ )$ is linear (i.e. $\dl_{\bfx}$ is
F-differentiable at $W$), and for $H$ such that $h_{ij_i}=u_i$,
$1\leq i\leq m$, and zero elsewhere, where we recall that in the
F-differentiable case, $j_i$ is the only index achieving the maximum
in $\phxi(W)$, i.e.  $\Ji = \lbrace j_i\rbrace$. Indeed in that case,
$\Vert H\Vert=1$ and $\dl_{\bfx}^\prime(W; H) = \mbf{u}$.

\subsubsection{Proposition of candidates $\Delta W$.}
\label{sec:proposition_candidates}

In the non-F-differentiable case (i.e. when at least one $\Ji$ has
more than one element), without analytically solving
\eqref{eq:Delta_W}, we can at least propose decent candidates,
inspired by the F-differentiable case. Let
$I^+ \coloneqq \lbrace 1\leq i\leq m, u_i \geq 0 \rbrace$,
$I^- \coloneqq \lbrace 1\leq i\leq m, u_i < 0 \rbrace$ and for
$1\leq i\leq m$ let us note $p_i\coloneqq |\Ji|$ the number of indices
achieving the maximum in $\phxi(W)$. To make $\dl_{\bfx}^\prime(W; H)$
similar to $\mbf{u}$ while keeping $\Vert H\Vert=1$, we propose:
\begin{itemize}
\item For $i\in I^+$, set $h_{ij_0}=u_i$ for any \emph{one}
  $j_0\in\Ji$, and zero for $j\neq j_0$
\item For $i\in I^-$, set $h_{ij}=\frac{u_i}{\sqrt{p_i}}$ for all
  $j\in\Ji$, and zero for $j\notin\Ji$.
\end{itemize}
Any such $H$ verifies $\Vert H\Vert=1$ and
\begin{equation}
  \label{eq:inner_product_proposition}
  \langle \dl_{\bfx}^\prime(W; H), \bfu \rangle = \sum_{i\in I^+}u_i^2 + \sum_{i\in I^-}\frac{u_i^2}{\sqrt{p_i}} = 1 - \sum_{i\in I^-}\left(1-\frac{1}{\sqrt{p_i}}\right)u_i^2.
\end{equation}
We see that this quantity gets closer to one as the $p_i$ get closer
to one, and we recover the optimal bound in the F-differentiable case,
which corresponds to $p_i=1$ for all $1\leq i\leq m$, or when all
$u_i\geq 0$. Furthermore, we have the lower bound
$\langle \dl_{\bfx}^\prime(W; H), \bfu \rangle \geq \frac{1}{\sqrt{n}}
> 0$, which is the left hand part of property
\eqref{eq:property_chain_rule}. Note that numerical experiments show
that better $H$ can be found (for example in the neighbourhood of the
proposed ones).

\subsubsection{Choosing the learning rate.}
\label{sec:learning_rate_param_update}
Recall that solving problem \eqref{eq:Delta_W} is relevant as long as
a good first order approximation
$\dl_{\bfx}(W+\eta H) \approx \dl_{\bfx}(W) + \eta\dl_{\bfx}^\prime(W;
H)$ holds, since only in this case does the parameter update ensure a
change in the output value towards a descent direction. Proposition
\ref{prop:affine_interval} provides the exact range of learning rates
for which this approximation is an equality. For our proposed $H$, it
holds if and only if $\eta\in[0, \epsilon]\cap\R^+$, with
  \begin{equation}
  \label{eq:epsilon_affine}
  \epsilon = \min_{i\in I^-} \frac{\eta_i \sqrt{p_i}}{|u_i|}
\end{equation}
where
$\eta_i \coloneqq \min\limits_{k\notin \Ji} \phxi(W) - (w_{ik} + x_k) =
\phxi(W) - \max\limits_{k\notin \Ji} (w_{ik} + x_k)$.

\subsection{Message passing}
\label{sec:message_passing}

\subsubsection{Problem setting}
\label{sec:pbm_setting_message_passing}

For the message passing, we are first faced with the same problem as
\eqref{eq:message_u_k} for F-differentiable functions, but with the
B-derivative. Namely, given the received target direction $\bfu$, we
want to find the best update direction $\Delta\bfx$ for $\bfx$,
\begin{equation}
  \label{eq:message_v_dil}
  \Delta\bfx = \arg\max_{\Vert\bfh\Vert = 1} \left\langle \dl_{W}^\prime(\bfx; \bfh), \bfu \right\rangle.  
\end{equation}
Assuming we can find a good enough $\bfh$, which would ensure
$\left\langle \dl_{W}^\prime(\bfx; \bfh), \bfu \right\rangle \geq 0$,
i.e. the right hand part of property \eqref{eq:property_chain_rule},
then we have another problem, which is to guarantee property
\eqref{eq:u_k_v_k_u_k1_v_k1}: that if
$\langle \bfv, \bfh\rangle \geq 0$ for some $\bfv$, then
$\left\langle \dl_{W}^\prime(\bfx; \bfv), \bfu \right\rangle \geq
0$. To make sure the chain rule works, we could therefore focus on the
problem
\begin{equation}
  \label{eq:problem_message}
  \begin{array}{cc}
    \text{Find } \bfh\in\R^n \text{ such that}
    &
      \left\lbrace
      \begin{array}{c}
        \Vert\bfh\Vert = 1\\
        \left\langle \dl_{W}^\prime(\bfx; \bfh), \bfu \right\rangle \geq 0\\ 
        \forall \bfv\in\R^n,\;\; \langle \bfv, \bfh\rangle \geq 0 \Rightarrow \left\langle \dl_{W}^\prime(\bfx; \bfv), \bfu \right\rangle \geq 0.       
      \end{array}\right.
  \end{array}
\end{equation}

\subsubsection{Proposition of candidates $\Delta \bfx$}
\label{sec:proposition_candidates_delta_x}

Recall that
$\dl_{W}^\prime(\bfx; \bfh) = \left( \max_{j\in \Ji} h_j
\right)_{1\leq i \leq m}$, hence contrary to the case of parameter
update (Section \ref{sec:updating_params}), the same $h_j$ can
contribute to different $\Ji$, which makes a heuristic construction of
$\bfh$ much more complicated. One exception is the case where the sets
$\Jicomp$ are pairwise disjoint, as with the max-pooling layer with
strides, for which the same kind of construction as in Section
\ref{sec:updating_params} can be done. However, this guarantees only
the first two conditions of \eqref{eq:problem_message}, but we cannot
say much about the last one.

At this stage we have no provable solution for
\eqref{eq:problem_message} except, obviously, in the F-differentiable
case, where each $\Ji$ is a singleton $\lbrace j_i\rbrace$. In that
case, as presented in Section~\ref{sec:derivative_x},
$\dl_{W}^\prime(\mbf{x} ; \mbf{h}) = E \mbf{h}$, where $E$ is the
$m\times n$ matrix whose rows are the $e_{j_i}$s, each $e_{j_i}$ being
the vector with a one at index $j_i$ and zeros elsewhere. Hence the
solution of \eqref{eq:message_v_dil}, and a solution of
\eqref{eq:problem_message}, is $\bfh = \frac{E^T\bfu}{\Vert E^T\bfu\Vert}$ if
$E^T\bfu\neq 0$, and any unit vector $\bfh$ otherwise.

Therefore we propose as update candidate, one that generalizes the
F-differentiable case, namely $\bfh = \frac{E^T\bfu}{\Vert E^T\bfu\Vert}$ but
with $E$ the matrix whose $i$-th row is
$E_{i,:} = \sum_{j\in\Ji} e_{j}$. Numerical experiments show that this
proposition can sometimes violate the last two conditions of
\eqref{eq:problem_message}, but often behaves well.

\subsubsection{Choosing the learning rate}
\label{sec:learning_rate_var_update}

Hoping that the chosen $\bfh$ fulfills \eqref{eq:problem_message}, we
make the best of it by choosing a learning rate ensuring the first
order equality \eqref{eq:dil_W_x_h}. Hence once again we follow the
construction inspired by Proposition \ref{prop:affine_interval}. The
choice of $\bfh = \frac{E^T\bfu}{\Vert E^T\bfu\Vert}$ yields no
simplification of the expression of $\epsilon$.

\subsection{The convolutional case}
\label{sec:convolutional_case}

The definitions \eqref{eq:dilation_layer} and \eqref{eq:erosion_layer}
cover translation invariant dilations and erosions, as soon as
$W \in\R^{n\times n}$ is a Toeplitz matrix.
%and $\bfx$ is one feature map.
However, in Section \ref{sec:updating_params}, we assumed no
``shared weights'', i.e. each row of $W$ was considered independent
from the others, which allowed an easy choice for the parameter
update.

To model the constraint on $W$ due to translation invariance, we
assume $\dl_W$ is represented by a vector $\bfw \in\R^p$, $p\leq n$,
and the input variable $\bfx\in\R^n$ is now seen as a matrix
$X\in\R^{n\times p}$ containing $n$ blocks of length $p$. The dilation
now writes
\begin{equation}
  \label{eq:dilation_layer_conv}
  \dl_{\bfw}(X) = \dl_{X}(\bfw) \coloneqq \left(\max_{1\leq j\leq p} x_{ij} + w_j\right)_{1\leq i\leq n}.
\end{equation}
Unfortunately, we see that even for the parameter update, which was
rather favorable in the ``dense'' layer case of Section
\ref{sec:updating_params}, we are in the same situation as in the
message passing of Section \ref{sec:message_passing}, in the sense
that finding good candidate for $\Delta\bfw$ is as difficult as
solving \eqref{eq:message_v_dil}. We would therefore apply the same
heuristics, i.e. $\bfh = \frac{E^T\bfu}{\Vert E^T\bfu\Vert}$, where
$E_{ij}=1$ if $j$ achieves the maximum in
$\max_{1\leq j\leq p} x_{ij} + w_j$ and zero elsewhere. Concerning the
learning rate, \eqref{eq:dil_W_x_h} holds.

\subsection{Practical consequences}
\label{sec:practical_consequences}

\textbf{Position in the network, dense or convolutional layer.} We
saw that the chain rule mechanism is not guaranteed with morphological
layers because of uncertainties in the message passing in general, and
even in the parameter update for convolutional operators. Therefore,
we expect better performance as a morphological layer is closer to the
input of the network, and even more so if it is a dense
layer. Typically, starting a neural pipeline with a dense dilation or
erosion is the most favorable case with the update and learning rate
propositions of Section \ref{sec:updating_params}. Furthermore, if
each morphological layer is seen as a noisy message transmitter, then
it is expected that many such layers in the same network may be hard
to train with the chain rule paradigm.\\
\textbf{Initialization} In both the dense and convolutional cases,
according to our propositions or even in the F-differentiable case, a
parameter coefficient is not updated if it does not contribute to a
maximum.  In the dense case, $w_{ij}$ is not modified if $j\notin \Ji$,
and in the convolutional one, $w_j$ remains unchanged if $j\notin \Ji$
for all $i$. In particular, if such coefficient is moved to $-\infty$,
it will never be updated anymore. Now, consider for example that if
the input variable $\bfx$ has values in $[0, 1]$ and at least one
weight $w_{ij_1} \geq 0$, then the closer another weight $w_{ij_2}$,
on the same line, will be to $-1$ the less likely it will be to achieve
the maximum, and $w_{ij_2}\leq -1$ is equivalent to
$w_{ij_2} = -\infty$. Therefore it seems preferable to initialize the
parameters with non-negative values (typically, zero if input values
in $[0, 1]$). Then, the proposed adaptive learning rates should avoid
a divergence of weights to values from where they cannot come back.

\section{Conclusion}
\label{sec:conclusion}
In this paper we investigated the optimization of morphological layers
based on the Bouligand derivative and the chain rule. We showed that
despite the first order approximation of the B-derivative, its
non-linearity makes morphological layers noisy message transmitter in
the chain rule, where they are not F-differentiable. We clearly stated
the problems to overcome in order to make this framework compatible
with the chain rule. We also provided insights regarding the choice of
the learning-rate for these layers, which seems much clearer than with
classic layers. Future work will deal with addressing the stated
problems and show the experimental consequences of the theoretical
results presented here.

\section*{Acknowledgments}
\label{acknowledgments}
I would like to thank Fran\c{c}ois Pacaud and Santiago Velasco-Forero for
fruitful discussions on this topic.

\bibliographystyle{splncs04}
\bibliography{refs}

\end{document}